\documentclass[letterpaper, 10 pt, conference]{ieeeconf}  

\IEEEoverridecommandlockouts                              

\usepackage{dsfont}
\usepackage{amsmath}
\usepackage{booktabs} 
\usepackage{multirow} 
\usepackage{algorithm}
\usepackage{algorithm}
\usepackage{algpseudocode}
\usepackage{amssymb}
\algtext*{EndFor}
\algtext*{EndIf}
\usepackage[table]{xcolor}
\usepackage{booktabs} 
\usepackage{array}
\usepackage{booktabs}
\usepackage{placeins}
\usepackage{graphicx}
\usepackage{caption}
\newtheorem{theorem}{Theorem}

\newtheorem{definition}{Definition}

\title{\LARGE \bf
MeshSIPP: Efficient Lattice Planning in Dynamic Environment
}

\author{Marat Agranovskiy and Konstantin Yakovlev}

\makeatletter
\newcommand{\preprinttext}[1]{\gdef\@preprinttext{#1}}

\let\oldmaketitle\maketitle
\renewcommand{\maketitle}{%
  \oldmaketitle
  \ifx\@preprinttext\undefined\else
    \bgroup
    \renewcommand\thefootnote{}
    \footnotetext{%
      \vspace{-12pt}
      \hrule width 0.4\columnwidth height 0.4pt 
      \vspace{5pt}
      \noindent \@preprinttext
    }%
    \egroup
  \fi%
}
\makeatother

\preprinttext{This is a preprint, currently under review.}

\begin{document}

\maketitle
\thispagestyle{empty}
\pagestyle{empty}

\begin{abstract}
Autonomous navigation in dynamic environments requires computing spatiotemporal trajectories that satisfy non-holonomic motion constraints. When the trajectories of the moving obstacles are predictable or known, a promising approach is to rely on the combination of state lattices constructed from precomputed feasible motion primitives and Safe Interval Path Planning -- a search-based algorithm with strong theoretical guarantees. While this approach yields feasible paths, the rich primitive sets needed for smooth navigation induce a large branching factor, which becomes costly when coupled with time-dependent obstacle intervals. To this end,
we present MeshSIPP, an efficient planner that removes the computational bottleneck by exploiting the fact that many primitives sweep the same regions and can therefore be validated together. MeshSIPP propagates primitives as spatial bundles, screens them with lightweight bounding-interval checks, and defers the expensive exact departure-time search until a primitive reaches its terminal state. A time-aware pruning rule additionally discards redundant space-time branches early in the search. We prove that the resulting search is complete and optimal. Extensive experiments over more than 6,000 benchmark instances and real-time ROS~2 simulations show that MeshSIPP achieves up to a 3$\times$ speedup over state-of-the-art spatiotemporal planners.
\end{abstract}


\section{Introduction}
Autonomous navigation in environments with dynamic obstacles is a fundamental challenge in mobile robotics. In many practical applications, such as warehouse logistics, the future movements of dynamic obstacles are either explicitly known or can be accurately predicted. Utilizing this predictive information enables a global spatiotemporal planning paradigm, allowing the robot to compute high-quality collision-free paths ahead of time. Beyond avoiding dynamic obstacles, practical deployment requires strict adherence to vehicle kinematics, as non-holonomic platforms (such as car-like or differential-drive vehicles) cannot turn in place and require smoothly curved, kinematically feasible trajectories.

A widespread approach to handle kinodynamic constraints in robotics is sampling-based path planning~\cite{rrt_star, rrt, sampling}. Planners of this family naturally operate in continuous spaces and rely on the randomized decomposition of the problem into smaller ones. 

They support complex vehicle constraints via customized steering functions and are particularly advantageous in high-dimensional configuration spaces, e.g., for robotic manipulators. Moreover, extensions like ST-RRT*~\cite{ST-RRT} incorporate a temporal dimension to avoid dynamic obstacles. Still, these methods provide only probabilistic guarantees of completeness and optimality and may struggle to find paths in confined configuration workspaces. 

Conversely, lattice-based planners rely on the discretization of the workspace/configuration space and provide strong theoretical guarantees with respect to this discretization. They are often preferable for systems with fewer degrees of freedom, such as mobile robots, where planning  primarily involves $(x,y,\theta)$. These methods reason over \emph{motion primitives} -- precomputed, kinodynamically feasible motions used to construct the final path (see Fig.~\ref{fig:planning_example}). Stacked motion primitives form a \textit{state lattice}, i.e., a graph where vertices correspond to the robot's states and edges to the feasible transitions between them. A shortest path on this graph can be obtained by the search algorithms like A*~\cite{astar}.

\begin{figure}[t]
  \centering
  \includegraphics[width=0.99\linewidth]{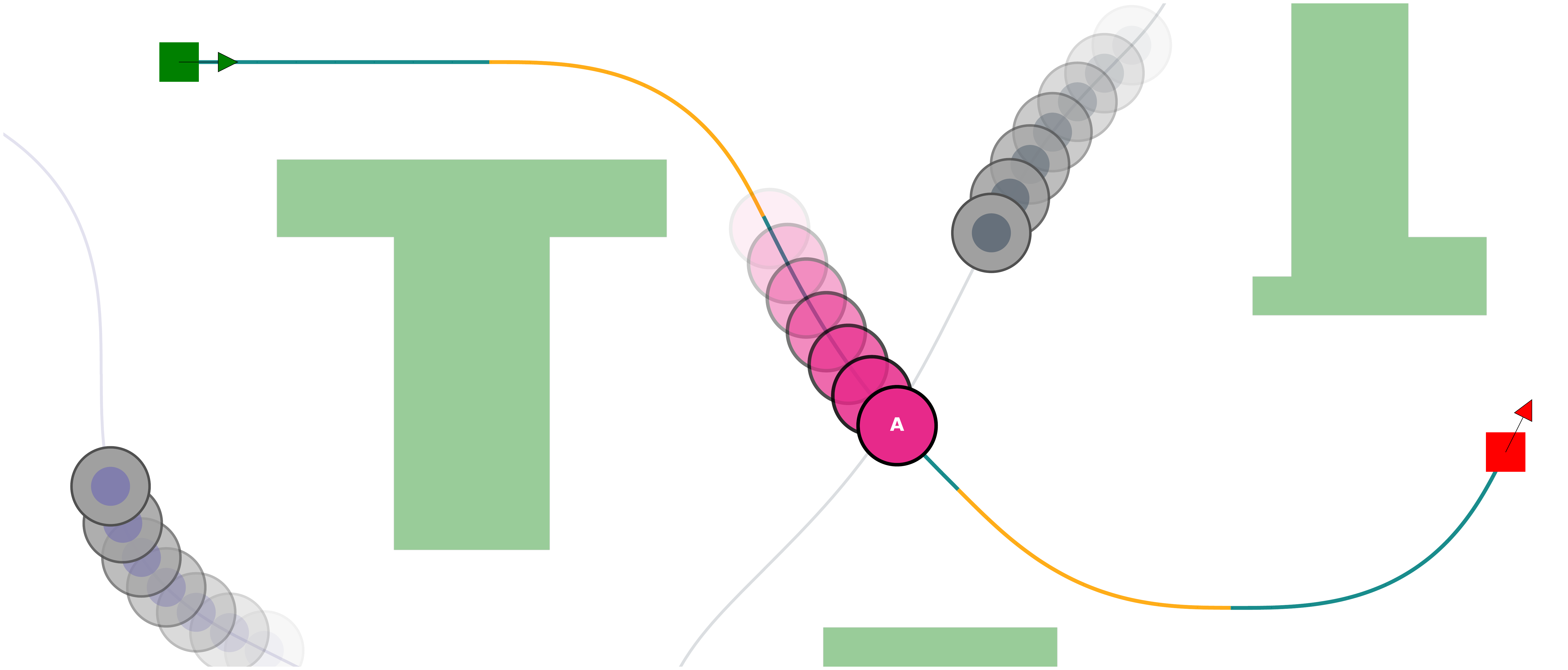}
  \captionsetup{font=small}
  \caption{Example of the path planning problem we are interested in. Within a 2D workspace (green: static obstacles, white: free space), the robot with state $(x, y, \theta)$ must compute a collision-free path composed of smoothly connected motion primitives (alternating colors) to reach the goal while avoiding dynamic obstacles.}
  \label{fig:planning_example}
  \vspace*{-12pt}
\end{figure}

To effectively extend this approach to dynamic environments, one can utilize Safe Interval Path Planning (SIPP)~\cite{sipp} that reduces the temporal search space by grouping obstacle-free time periods into contiguous safe intervals. This abstraction has been successfully extended to any-angle pathfinding~\cite{any_angle_sipp}, strict acceleration profiles~\cite{kinodynamic_sipp}, quadrotors~\cite{quadrotor_sipp}, high-dimensional manipulators~\cite{manipul_sipp}, anytime planning~\cite{anytime_sipp}, etc.

However, directly integrating state lattices with SIPP, as in~\cite{LatticeSIPP}, introduces a severe computational bottleneck. To produce smooth and expressive trajectories, rich motion-primitive libraries are required and this leads to a large branching factor. In dynamic environments, this is further amplified by the fragmentation of safe intervals at each state, resulting in a multiplicative increase in search complexity. As a consequence, straightforward combinations of lattice-based and safe interval planning become computationally burdensome in complex, obstacle-rich scenarios.

To this end, we introduce \textbf{MeshSIPP}, a novel kinodynamic spatiotemporal planning algorithm that transforms the conventional state lattice into a specialized \textit{mesh graph}~\cite{mesha_star}. Because many primitives sweep the same areas, MeshSIPP groups them into coherent bundles and propagates each bundle incrementally along their shared footprint, establishing spatial validity once for the whole group. Extending this idea to dynamic environments is not straightforward: primitives within a bundle have different execution durations, so evaluating exact temporal constraints during propagation would fragment the bundle and destroy the structure that makes it efficient. We resolve this with a lazy evaluation scheme that defers the expensive continuous departure search until a primitive reaches its terminal state, combined with a time-aware pruning rule that discards redundant space-time branches early in the search.

Our contributions are: (i) MeshSIPP, a bundle-based safe-interval search with lazy temporal evaluation; (ii) a bounded, time-aware terminal pruning rule, together with a proof that it preserves completeness and optimality with respect to the given workspace discretization and control set; (iii) an evaluation over 6,000+ instances against four baselines, in which MeshSIPP returns solutions of identical cost to SIPP at a 100\% success rate while running $2-3\times$ faster; and (iv) a real-time ROS/Gazebo validation on a TurtleBot3, replanning in 50 ms per cycle amidst 20 non-cooperative moving obstacles.

\section{Problem Statement}

Consider a mobile robot, approximated by a bounding circle of radius $R$, operating in a planar 2D workspace $\mathcal{W} \subset \mathbb{R}^2$. Its continuous state is defined by its coordinates and heading angle: $\mathbf{s} = (x, y, \theta) \in \mathcal{W} \times S^1$. The underlying motion is subject to non-holonomic kinematic constraints expressed through the dynamic model $\dot{\mathbf{s}} = f(\mathbf{s}, \mathbf{u})$, where $\mathbf{u}$ denotes the bounded control inputs. As is typical for mobile platforms, these dynamics are translationally invariant, being independent of the absolute spatial coordinates $(x, y)$.

\textbf{State Lattice and Motion Primitives}. To formulate continuous planning as a graph search, we define a discrete set of headings $\Theta = \{\theta_1, \dots, \theta_k\} \subset S^1$ and a finite \textit{control set} of \textit{template motion primitives}. Each template specifies an initial heading $\theta \in \Theta$, a control sequence $\mathbf{u}(t)$, and an execution duration $\Delta t > 0$. Due to the translational invariance of the dynamics, integrating $\dot{\mathbf{s}} = f(\mathbf{s}, \mathbf{u})$ from an initial state $(x, y, \theta)$ over $\Delta t$ produces a kinematically feasible trajectory segment whose continuous geometric shape is independent of the starting coordinates $(x, y)$. Consequently, rather than reasoning about time-varying control inputs, these templates can be treated directly as precomputed, invariant trajectory segments that the robot can execute from any spatial location.

Instantiating these templates at a common origin, such as $(0,0)$, yields a finite set of relative endpoint displacements that can be embedded into a regular spatial grid (see Fig.~\ref{fig:base_control_set}). This grid is not unique, as any of its refinements preserve the alignment of all primitive endpoints. We choose a resolution that reflects the robot's sensing and localization capabilities, using the resulting grid as the environmental \textit{occupancy map}.

\begin{figure}[t]
  \centering
  \includegraphics[width=0.99\linewidth]{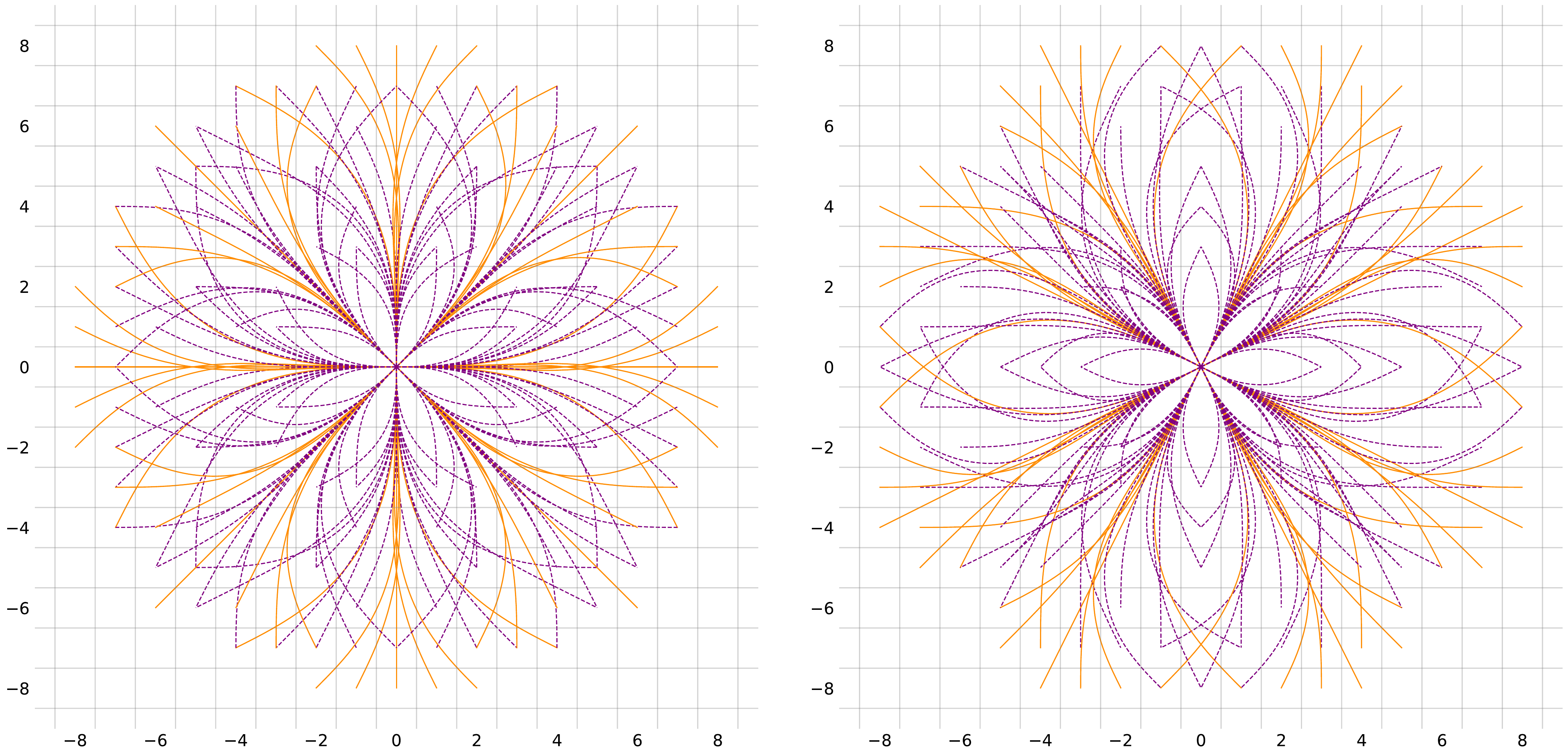}
  \captionsetup{font=small}
  \caption{Control set with regular grid: totaling $384$ primitives across
           $16$ discrete orientations. Two colors are used for visual clarity.}
  \label{fig:base_control_set}
\end{figure}

To close this construction under transitions, we additionally design each template so that its terminal heading also belongs to $\Theta$. We can therefore define the discrete state space as $\mathcal{S}=\{(i,j,\theta)\mid (i,j)\in\mathbb{Z}^2,\ \theta\in\Theta\}$, where $(i, j)$ corresponds to the cell centers of the chosen grid. A template with initial heading $\theta$ can be instantiated at any state $(i,j,\theta)$. Because its relative displacement is grid-aligned and its terminal heading belongs to $\Theta$, the resulting trajectory terminates at another state in $\mathcal{S}$. Thus, template instantiations define directed edges between discrete states and form a regular \textit{state lattice}. A path in this graph represents a sequence of smoothly connected primitives, yielding a continuous, kinematically feasible global trajectory whose expressiveness scales with the richness of the control set.

\textbf{Environment Representation}. We represent the environment on the same occupancy grid introduced above, and distinguish between static and dynamic obstacles:
\begin{itemize}
    \item \textit{Static Environment}: A binary occupancy map $M$, where $M_{i, j} = 1$ if cell $(i, j)$ is blocked, and $0$ otherwise.

    \item \textit{Dynamic Environment}: A spatio-temporal occupancy map $D$, where each cell $(i,j)$ is associated with a set $D_{i,j}=\{[t_{in}^{(1)},t_{out}^{(1)}), [t_{in}^{(2)},t_{out}^{(2)}), \dots\}$  of disjoint intervals, indicating when the cell is occupied by a moving obstacles. These intervals are precomputed by dilating dynamic obstacles by the robot radius $R$, sweeping them along their trajectories, and rasterizing the resulting spatio-temporal volumes onto the grid.
\end{itemize}

\textbf{Spatial-Temporal Trace}. As the robot executes a template primitive, its center sequentially traverses a set of grid cells. For each visited cell $c$, we record its coordinate offset $\Delta c=(\Delta i_c,\Delta j_c)$  from the primitive's origin and the relative time interval $[\tau_{in}^c,\tau_{out}^c)\subseteq[0,\Delta t)$ during which the center resides within $c$. When the primitive is instantiated at $(i,j,\theta)$ with absolute start time $t$, this local trace shifts accordingly. The execution is \textit{collision-free} if all cells in the trace are statically free in $M$ and its absolute intervals do not overlap with any obstacle intervals: $[t+\tau_{in}^c,t+\tau_{out}^c)\cap D_{i+\Delta i_c,\,j+\Delta j_c}=\emptyset$.

\textbf{Objective}. A valid path is a sequence of collision-free motion primitives, optionally interleaved with waits in safe intervals. Each action must start from the discrete state reached by the previous one, ensuring kinematic continuity. Starting from $s_0=(i_0,j_0,\theta_0)$ at time $t_0$, we seek a path to $s_f=(i_f,j_f,\theta_f)$ that minimizes the arrival time $t_f$.

\section{Background: Safe Interval Path Planning}

To effectively navigate dynamic environments where in-place waiting is permitted to avoid obstacles, several methods exist. Standard space-time search algorithms operate over an augmented state space $(x, y, \theta, t)$. By modeling waiting actions through discrete time increments ($t \to t + \Delta t$), these methods often suffer from severe state-space explosion, especially over long planning horizons or in continuous time. Safe Interval Path Planning (SIPP) \cite{sipp} mitigates this inefficiency by abstracting the continuum of free time at each spatial configuration into a finite set of maximal, contiguous, collision-free windows called \textit{safe intervals}. Over this interval-augmented space, a search state is defined as $p = \langle i, j, \theta, I \rangle$, and SIPP  functions similarly to A* (or Dijkstra's algorithm) to find a path to each state $p$ that minimizes the arrival time $g(p) = t_{arr} \in I$. This formulation enables a \textit{dominance property}: if multiple paths reach the same spatial configuration within the same safe interval $I$, only the earliest arrival time $t_{arr}$ needs to be maintained. Any later one is inherently dominated, as the robot could instead arrive earlier and safely wait in place within $I$ until any later moment. This allows SIPP to prune redundant trajectories and bound the search space by the number of safe intervals, regardless of the amount or duration of wait actions.

Consequently, a single spatial configuration gives rise to multiple search states, each corresponding to different time windows (before, between, or after the passage of dynamic obstacles). Maintaining all of them is essential for completeness and optimality, as arriving at an intermediate configuration in a later interval may be the only way to access a downstream safe interval that yields a faster overall path to the goal. Given this representation, the SIPP search proceeds by iteratively expanding states $p = \langle i, j, \theta, I \rangle$ to generate successors. To do so, all applicable motion primitives $m$ from the current configuration are evaluated. For each target configuration $(i', j', \theta')$ resulting from $m$, the algorithm retrieves the list of available safe intervals. For every target interval $I'$, we compute the earliest possible departure time $t_{\text{dep}}$ from the current state --- which directly minimizes the arrival time in $I'$ --- such that the robot successfully departs from the current interval, arrives within the target interval, and avoids all dynamic obstacles along the path (see Fig.~\ref{fig:departure-and-mesh}, left). Formally, a valid departure time $t_{\text{dep}}$ must satisfy:
\begin{enumerate}
\item $t_{dep} \ge t_{arr}$, enforcing temporal causality;
\item $t_{dep} + m.\tau_{out}^{(0,0)} \le I.t_{end}$, ensuring the agent vacates the origin cell before the current interval $I$ closes, where $m.\tau_{out}^{(0,0)}$ is the time required to leave the cell;
\item $t_{dep} + \Delta t_{m} \in I'$, ensuring arrival at the target cell within the destination interval $I'$, where $\Delta t_{m}$ is the execution duration of primitive $m$.
\end{enumerate}
This step generates a successor state $p' = \langle i', j', \theta', I' \rangle$ with an updated arrival time $g(p') = t_{dep} + \Delta t_m$

We encapsulate the calculation of the time $t_{dep}^*$ within a function $\textsc{GetEarliestDeparture}((i,j), m, t_{min}, t_{max})$. While its specific implementation is environment-dependent, for a grid-based state space, it can be achieved by projecting all temporal constraints back to the moment of departure. Specifically, the function iterates over all cells traversed by the primitive $m$, retrieves the occupied intervals of dynamic obstacles in those cells, and projects them backward in time according to the primitive's spatial-temporal trace. This projection yields a set of forbidden departure time blocks at the origin cell. By merging and sorting these invalid blocks, the algorithm identifies the earliest free gap that falls within the allowable departure window $[t_{min}, t_{max}]$.

\begin{figure}[t]
  \centering
  \includegraphics[width=0.55\linewidth]{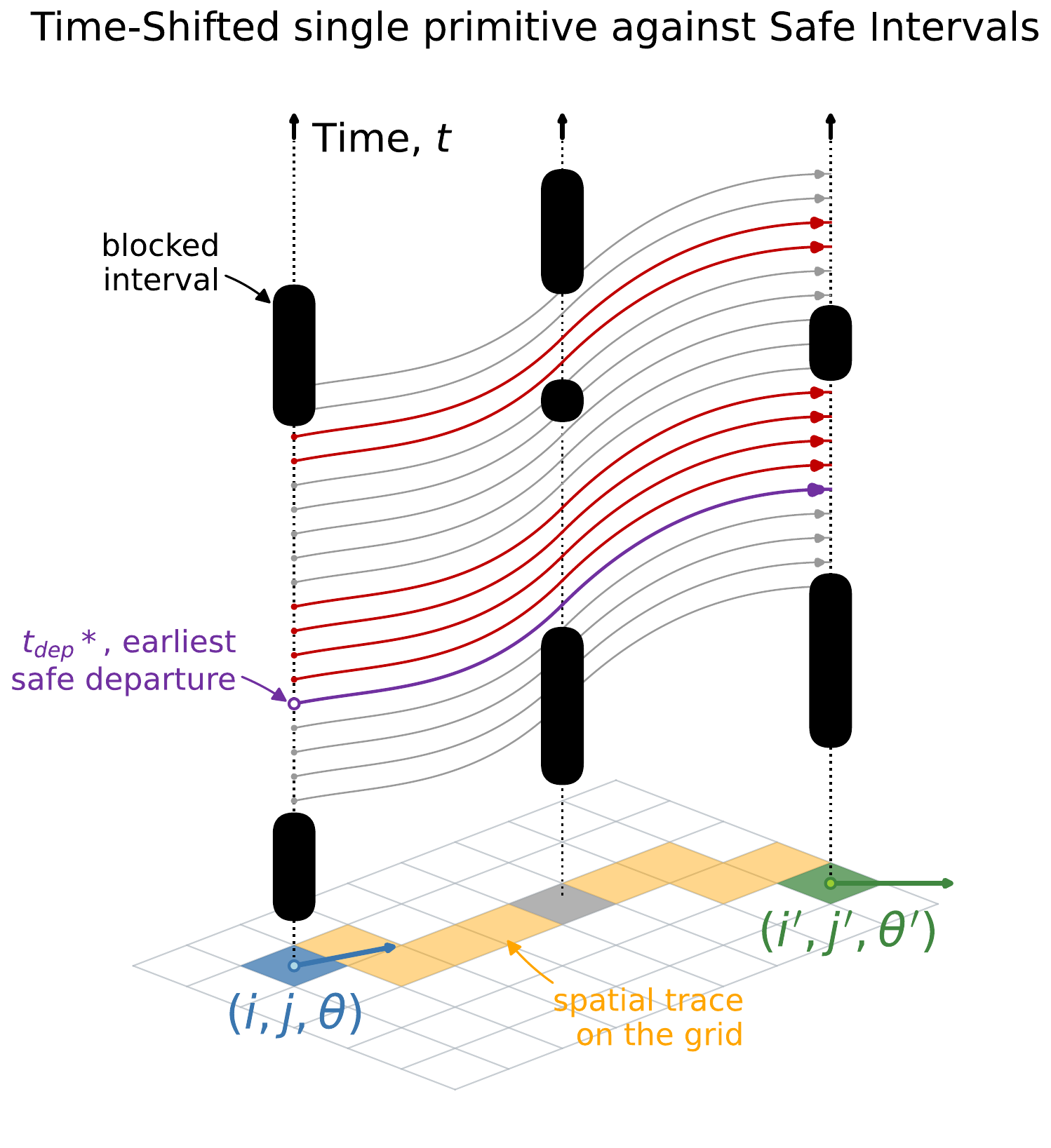}\hfill
  \includegraphics[width=0.35\linewidth]{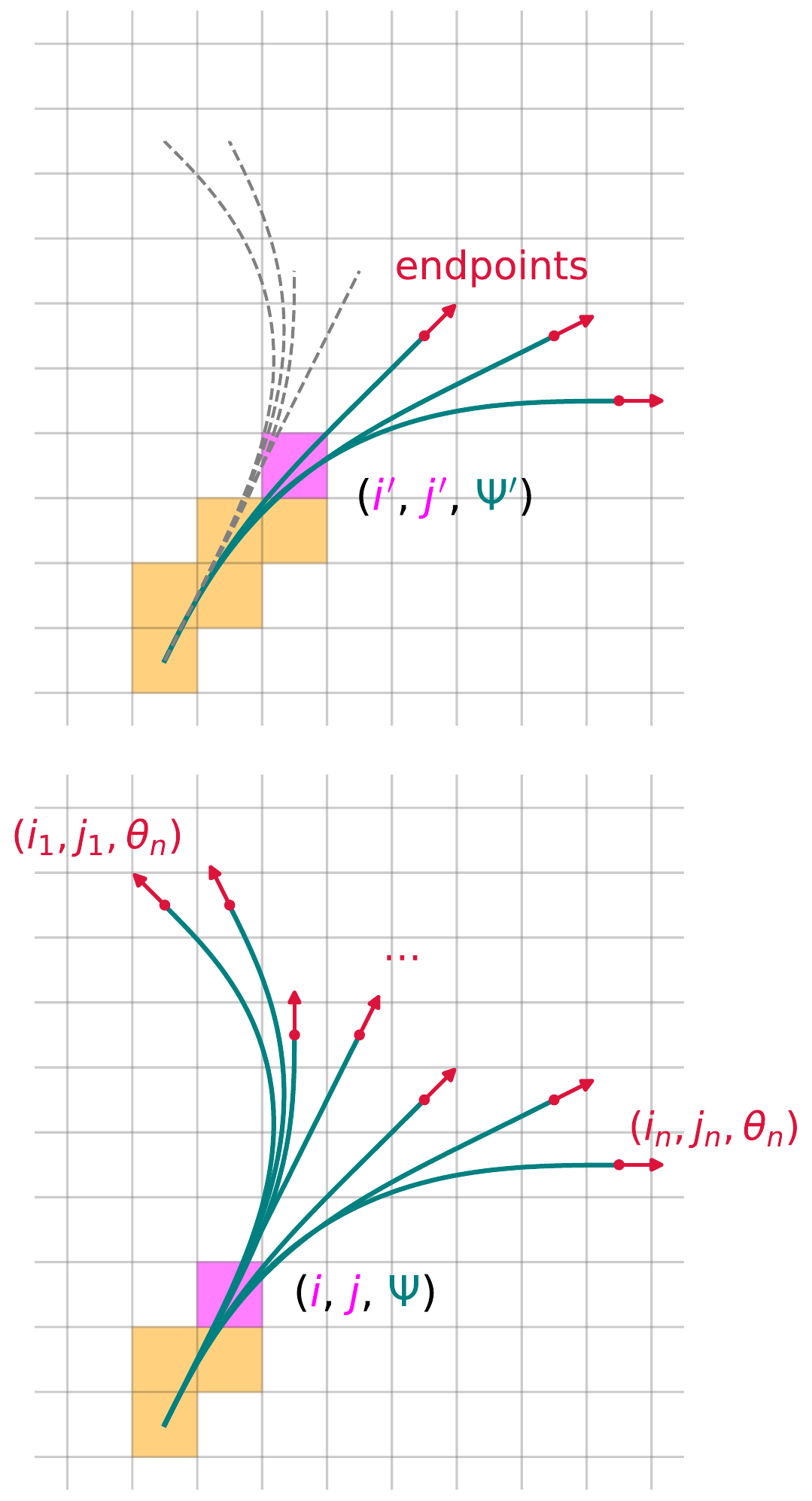}
  \captionsetup{font=small}
  \caption{\textbf{Left:} Finding the optimal departure time $t_{\text{dep}}^*$ for a single motion primitive. Colored curves represent the same primitive shifted across continuous times: gray trajectories collide with dynamic obstacle intervals (black blocks), red options are safe, and the purple trajectory denotes the earliest safe departure we are interested in. 
  \textbf{Right:} Extended cells in the mesh graph, representing a grid cell (magenta) combined with a bundle of motion primitives (teal) passing through it at a specific $0$-based spatial trace index $k$. The bottom panel depicts $(i, j, \Psi)$ at index $k=3$. The top panel shows its second-order successor $(i', j', \Psi')$ at $k=5$, retaining only the primitives whose spatial traces match the bundle's exact sequence of cells (orange ones) up to this point. The remainder are pruned (gray dashed lines) as their traces diverged at an earlier step. Red arrows indicate the endpoints of the surviving primitives in the bundle.}

  \label{fig:departure-and-mesh}
\end{figure}

\section{Method: MeshSIPP}

The proposed method combines the interval-based reasoning of SIPP with a structural transformation of the search space (state lattice) into a specialized \textit{mesh graph}. Conducting a search directly over this representation enables early pruning and mitigates redundant state expansions. 

\subsection{Mesh Graph}

To reason about motion primitives at the individual grid-cell level, we adopt the \textit{mesh graph} formulation introduced in~\cite{mesha_star}. In conventional state-lattice planning, search nodes represent discrete states $(i, j, \theta)$ and edges represent full motion primitives. In contrast, the mesh graph decomposes the primitive execution cell-by-cell along its spatial trace, allowing primitives with overlapping footprints to be evaluated simultaneously in bundles within the same search node.

Formally, nodes in the mesh graph are \textit{extended cells}, defined as tuples $u = (i, j, \Psi)$, where $(i, j)$ specifies the 2D grid cell and $\Psi$ represents a \textit{primitive configuration}:
\begin{equation*}
\Psi = \left\{ (m_1, k), (m_2, k), \dots, (m_n, k) \right\}.
\end{equation*}
Here, each pair $(m, k) \in \Psi$ indicates that an instance of the template primitive $m$ (from a precomputed control set) traverses cell $(i, j)$ as the $k$-th cell ($0$-based index) in its spatial trace. Thus, $\Psi$ captures the information about a specific bundle of primitives simultaneously passing through cell $(i, j)$ at identical spatial offsets $k$ (see Fig.~\ref{fig:departure-and-mesh}).

To establish equivalence with the conventional state lattice, extended cells are categorized into two types:
\begin{itemize}
    \item \textit{Initial Extended Cells:} Characterized by a configuration $\Psi_\theta = \left\{ (m_1, 0), \dots, (m_r, 0) \right\}$ containing all control-set primitives originating at orientation $\theta$ with step index $k=0$. An initial cell $(i, j, \Psi_\theta)$ serves as the direct functional counterpart to a standard lattice state $(i, j, \theta)$.
    \item \textit{Regular Extended Cells:} Others: intermediate cells ($k > 0$) along the spatial traces of active primitives from a bundle, as depicted in Fig.~\ref{fig:departure-and-mesh} (right).
\end{itemize}

Transitions in the mesh graph propagate active bundles cell-by-cell along the spatial traces of their constituent primitives. To generate a successor configuration for an immediate neighbor, the algorithm filters the current bundle $\Psi$: it retains only the primitives whose next swept cell matches that neighbor, while primitives routing elsewhere are excluded. For instance, in Fig.~\ref{fig:departure-and-mesh} (right, bottom panel), the bundle at trace index $k=3$ physically diverges. If the search follows the rightward branch and advances two steps, it reaches a second-order successor $(i', j', \Psi')$ at $k=5$ (top panel). The primitives that routed along the alternate branch are inherently pruned from $\Psi'$ (illustrated by gray dashed lines). This local filtering mechanism ensures that primitive bundles remain grouped as long as they share a spatial path, splitting only when their traces physically diverge.

When a primitive $m \in \Psi$ completes its spatial trace, it reaches a specific discrete endpoint state $(i_m, j_m, \theta_m)$ (indicated by red arrows in Fig.~\ref{fig:departure-and-mesh}, right). Reaching this terminal state initiates a new bundle, meaning the succeeding node is represented as an initial extended cell $(i_m, j_m, \Psi_{\theta_m})$. Edge costs in the mesh graph are assigned as follows:
\begin{itemize}
    \item Transitions into \textit{regular} extended cells carry a cost $0$.
    \item Transitions into \textit{initial} extended cells are assigned the full cost of the completed primitive, $\text{Cost}(m)$. Exactly one primitive completes upon entering any given initial cell, ensuring this edge cost is uniquely defined.
\end{itemize}
Formally, we encapsulate the entire neighbor-generation process in a function $\textsc{MeshSuccessors}(i, j, \Psi)$, which yields the set of all valid successor extended cells $(i', j', \Psi')$.

Crucially, paths between initial extended cells are in bijection with motion-primitive sequences on the state lattice, preserving both total cost and exact spatial trace.

\subsection{MeshSIPP: Lazy Temporal Evaluation}

Integrating continuous temporal constraints directly into the spatial mesh graph requires careful handling of primitive execution times. Because primitives within a single bundle possess distinct durations $\Delta t$, evaluating exact dynamic collisions cell-by-cell would prematurely fragment the bundles, diminishing the core efficiency of the mesh graph structure. To preserve the structural advantage of simultaneous primitive evaluation, we introduce a \textit{lazy temporal evaluation} scheme. This approach decouples spatial traversal from strict dynamic interval validation: spatial propagation along regular cells utilizes lightweight bounding checks, while the exact continuous departure search is deferred until primitives reach their endpoints at initial cells.

\textbf{Search Space Elements.} To support this delayed evaluation, we augment the spatial extended cells of the mesh graph with the temporal structure of SIPP.

\begin{definition}[MeshSIPP State]
A state in the MeshSIPP search space is defined as a tuple $u = \langle i, j, \Psi, s_{\text{root}} \rangle$, where $(i, j, \Psi)$ is the spatial extended cell, and $s_{\text{root}} = \langle i_0, j_0, \theta_0, I_0 \rangle$ is the base SIPP state from which the current bundle $\Psi$ originated. Thus, $(i_0, j_0, \Psi_{\theta_0})$ corresponds to the last initial extended cell traversed on the path up to $u$. The $g$-value (current time) of the state $u$ is defined entirely by the arrival time of its root state, i.e., $g(u) = g(s_{\text{root}}) = t_{\text{arr}}$.
\end{definition}

In this augmented space, the root state $s_{\text{root}}$ acts as a temporal anchor. As the bundle propagates through intermediate cells, time remains logically anchored at $t_{\text{arr}}$, and transition costs are inherently zero. Time advances and costs accumulate exclusively when a primitive terminates and a new temporal anchor is instantiated.

\textbf{Successor Generation and Conservative Pruning.} Algorithm~\ref{algorithm_meshsipp_successors} details the successor generation process. For a state $u$, the algorithm first retrieves the spatial successors from the underlying mesh graph via $\textsc{MeshSuccessors}$ (Line 3).

If a transition yields a \texttt{REGULAR} cell (Line 4), we apply \textit{conservative temporal pruning} (Lines 5--8). We estimate a bounding window $[t_{\text{min}}, t_{\text{max}})$ representing the earliest possible arrival and latest possible departure of \textit{any} primitive within the bundle $\Psi'$ traversing the current cell. If dynamic obstacles ($D_{i'j'}$) completely subsume this window, the transition is pruned. Otherwise, it forms a new regular state with zero transition cost (Lines 9--10).

Conversely, reaching an \texttt{INITIAL} successor (past Line 11) means a primitive $m$ has completed, triggering the deferred exact evaluation. For each destination safe interval $I'$, we compute a bounded departure window $[t_1,t_2]$ relative to $s_{\text{root}}$ (Lines 18--19) and execute the continuous departure search (Line 20). If a valid departure time $t_{\text{dep}}$ is found, we establish a new temporal anchor $s_{\text{new}}$ and assign the true transition cost as the required wait time at the root cell plus the primitive's execution duration: $(t_{\text{dep}} - t_{\text{arr}}) + m.\Delta t$.

\begin{algorithm}[ht]
\caption{Generating Successors of a MeshSIPP State}
\label{algorithm_meshsipp_successors}
\textbf{Input:} MeshSIPP state $u = \langle i, j, \Psi, s_{root} \rangle$ with time $t_{arr}$\\
\textbf{Output:} Set of successors and transition costs $(v, \mathit{cost})$
\begin{algorithmic}[1]
\State $\texttt{Successors} \gets \emptyset$

\ForAll{$\langle i', j', \Psi' \rangle \in \textsc{MeshSuccessors}(i, j, \Psi)$}

\If{$\Psi'$ is \texttt{REGULAR}}
    \State $t_{min} = t_{arr} + \min_{(m, k) \in \Psi'} (\tau_{in}^k)$
    \State $t_{max} = t_{arr} + \max_{(m, k) \in \Psi'} (\tau_{out}^k)$ 
    \If{$[t_{min}, t_{max}) \subset I_{b}$ for some $I_{b} \in D_{i'j'}$}
        \State \textbf{continue} \Comment{Fully covered by obstacles}
    \EndIf
    \State $v_{int} \gets \langle i', j', \Psi', s_{root} \rangle$
    \State $\texttt{Successors}.\texttt{add}\left( (v_{int}, 0) \right)$
    \State \textbf{continue}
\EndIf

\State $\theta' \gets$ heading from which $\Psi'$ originates \Comment{$\Psi' = \Psi_{\theta'}$}
\State $m \gets$ the primitive that just terminated to form $\Psi'$
\State $p \gets (s_{root}.i,~s_{root}.j)$  \Comment{Origin cell}
\State $\tau^* \gets m.\tau_{out}^{(0,0)}$ \Comment{Time to clear origin}
\State $\texttt{SafeIntervals} \gets \mathds{R}_{\ge 0} \setminus D_{i', j'}$

\ForAll{$I' = [t_{start}, t_{end}) \in \texttt{SafeIntervals}$}
    \State $t_{1} \gets \max(t_{arr},~t_{start} - m.\Delta t)$
    \State $t_{2} \gets \min(s_{root}.I.t_{end} - \tau^*,~t_{end} - m.\Delta t)$
    \State $t_{dep} \gets \textsc{GetEarliestDeparture}(p, m, t_{1}, t_{2})$
    \If{$t_{1} < t_{2}$ and $t_{dep} \ne \infty$}
        \State $s_{new} \gets \langle i', j', \theta', I' \rangle$
        \State $v_{term} \gets \langle i', j', \Psi', s_{new} \rangle$ 
        \State $t_{arr}^{new} \gets t_{dep} + m.\Delta t$  \Comment{$= g(s_{new})$}
        \State $\mathit{cost} \gets (t_{dep} - t_{arr}) + m.\Delta t$
        \State $\texttt{Successors}.\texttt{add}\left( (v_{term}, \mathit{cost}) \right)$
    \EndIf
\EndFor
\EndFor
\State \Return $\texttt{Successors}$
\end{algorithmic}
\end{algorithm}

\textbf{Heuristic.} As in other informed search algorithms, we use a heuristic to estimate the remaining cost to the goal and focus the search on promising states. Let $\widehat{h}_{\mathrm{SIPP}}$ denote an admissible SIPP heuristic; as is common, it may ignore the temporal component and estimate only the remaining motion cost. For an initial MeshSIPP state, we use the heuristic value of its corresponding SIPP state. For a regular state $u=\langle i,j,\Psi,s_{\mathrm{root}}\rangle$, each $(m,k)\in\Psi$ defines a discrete endpoint state $s_m$ obtained by completing $m$. Since the cost of a primitive is incurred only at its endpoint, we set
\[
h(u)=\min_{(m,k)\in\Psi}
\left\{\Delta t_m+\widehat{h}_{\mathrm{SIPP}}(s_m)\right\}.
\]
If $\widehat{h}_{\mathrm{SIPP}}$ also accounts for safe intervals, the minimum is additionally taken over all reachable safe intervals of the primitive endpoints.

\subsection{Early Terminal Pruning in MeshSIPP}

A profound advantage of transitioning to a cell-by-cell search space (mesh graph) is the structural capability to prune unpromising branches prematurely. In a static environment, a regular search node can be safely discarded if all its spatial endpoints are already expanded. Since any valid path from this node must pass through one of these endpoints, its further exploration is redundant. Extending this principle to SIPP, however, introduces a temporal challenge: a single spatial endpoint $(i, j, \theta)$ encompasses multiple safe intervals.

\begin{algorithm}[ht]
\caption{MeshSIPP Early Terminal Pruning}
\label{algorithm_meshsipp_prune}
\textbf{Input:} Regular $u = \langle i, j, \Psi, s_{root} \rangle$ with arrival time $t_{arr}$\\
\textbf{Output:} \texttt{true} if $u$ is pruned, \texttt{false} otherwise
\begin{algorithmic}[1]
\State $\Delta t_{max} \gets \max\limits_{m \in \texttt{ControlSet}} (m.\Delta t)$ \Comment{Slowest primitive}
\vspace{0.25\baselineskip}
\ForAll{endpoint $(i', j', \theta')$ for each primitive $m \in \Psi$}
\State $t_{min} \gets t_{arr} + m.\Delta t$  \Comment{$W_m  = [t_{min}, t_{max}]$}
\State $t_{max} \gets s_{root}.I.t_{end} - m.\tau_{out}^{(0,0)} + m.\Delta t$
\State $\mathcal{I} \gets \left(\mathds{R}_{\ge 0} \setminus D_{i',j'}\right) \cap [t_{min}, t_{max}]$ \Comment{Reachable}
\If{$\mathcal{I}$ is empty}
    \State \textbf{continue} \Comment{Arrival within $W_m$ is blocked}
\EndIf

\State $I_{last} \gets \text{last interval in } \mathcal{I}$ \Comment{Nearest to $t_{max}$}

\If{\textbf{not} \textsc{IsExplored}$\left(\langle i', j', \theta', I_{last} \rangle\right)$}
    \State \Return \texttt{false} \Comment{Found unexplored interval}
\EndIf

\State $t_{cutoff} \gets \min(t_{max},~t_{arr} + \Delta t_{max})$
\State $\mathcal{I}_{0} \gets$ intervals in $\mathcal{I}$ before $t_{cutoff}$

\ForAll{$I \in \mathcal{I}_{0}$ (checked backward)}
    \If{\textbf{not} \textsc{IsExplored}$\left(\langle i', j', \theta', I \rangle\right)$}
        \State \Return \texttt{false}
    \EndIf
\EndFor

\ForAll{$I \in \mathcal{I} \setminus  \mathcal{I}_{0} \setminus \{I_{last}\}$ (checked backward)}
    \If{\textbf{not} \textsc{IsExplored}$\left(\langle i', j', \theta', I \rangle\right)$}
        \State \Return \texttt{false}
    \EndIf
\EndFor
\EndFor

\State \Return \texttt{true} \Comment{All reachable intervals are explored}
\end{algorithmic}
\end{algorithm}

To maintain theoretical optimality without exhaustive global interval checks, we introduce a bounded, time-aware pruning mechanism. For a regular state $u = \langle i, j, \Psi, s_{\text{root}} \rangle$ with arrival time $t_{\text{arr}}$, the physically possible arrival times at any spatial endpoint via a primitive $m$ are strictly confined to a continuous window $W_m = [t_{\text{min}}, t_{\text{max}}]$. It reflects the constraint that motion cannot initiate before the agent's arrival -- $t_{\text{min}} = t_{\text{arr}} + m.\Delta t$ -- or later than the interval's closure, accounting for the time required to fully vacate the root grid cell: $t_{\text{max}} = s_{\text{root}}.I.t_{\text{end}} - m.\tau_{\text{out}}^{(0,0)} + m.\Delta t$.

To safely prune $u$, every safe interval $I$ intersecting $W_m$ must be \textit{explored} -- meaning its corresponding SIPP state has already been expanded by the search (which, under A* properties, guarantees its optimal cost is securely fixed), or its current arrival time matches the theoretical absolute minimum of the interval, $I.t_{\text{start}}$. If even one intersecting interval remains unexplored, the endpoint might yield a better path, and pruning is immediately aborted.

\begin{theorem}
MeshSIPP is complete and optimal, yielding identical solutions (in terms of path cost and spatial traces) to the standard SIPP algorithm over the same control set.
\end{theorem}

\textit{Proof Sketch:} The proof establishes that augmenting the mesh graph with temporal anchors ($s_{\text{root}}$) and deferred continuous evaluation yields a dynamic search space isomorphic to the standard SIPP state lattice with respect to optimal paths. We first prove the correspondence for a single motion primitive. Specifically, MeshSIPP constructs a chain of states whose grid projections form exactly the primitive's spatial trace, ensuring identical interactions with static obstacles in both representations. Once the trace reaches the primitive's endpoint, MeshSIPP performs the exact continuous temporal evaluation, as in SIPP, yielding the same departure window, successor safe interval, and transition cost. This correspondence extends to arbitrary sequences of primitives by concatenating the corresponding MeshSIPP chains. Finally, the pruning rules do not break this correspondence: conservative pruning removes only branches whose entire reachable temporal window is occupied (Algorithm~\ref{algorithm_meshsipp_successors}, Line 7), while early terminal pruning removes only branches whose reachable intervals have already been explored.

\subsection{Heuristic Interval Ordering}

While verifying all intervals within $W_m$ guarantees optimality, a sequential scan introduces a computational bottleneck. To accelerate the discovery of unexplored intervals -- and thereby abort the pruning check faster -- we evaluate the intervals in a heuristically permuted order. 

Our strategy is designed to \textit{fail-fast} by exploiting the chronological expansion properties of A*. Since later time intervals are typically explored deeper in the search tree, scanning reachable intervals in reverse chronological order maximizes the probability of immediately encountering an unexplored state. Furthermore, a powerful bound can be established based on the expansion status of the latest reachable interval. If $I_{\text{last}}$ (nearest to $t_{\text{max}}$) \textit{has already} been explored, it implies that the search previously reached this spatial cell via some primitive $m$ from a predecessor state expanded at some earlier time $t' \le t_{\text{arr}}$. The same spatial trajectory with zero (or minimal) waiting would have reached earlier safe intervals at arrival times starting from $t' + \Delta t$. Since $t' \le t_{arr}$, we can strictly bound these early arrival times by defining a cutoff threshold: $t' + \Delta t \le t_{arr} + \max_{m \in \texttt{ControlSet}}(\Delta t) =: t_{cutoff}$. Consequently, all intervals from $t_{max}$ down to $t_{cutoff}$ are highly likely to have already been expanded via these zero-wait continuations, unless a dynamic obstacle has blocked those earlier arrivals.

This logic is implemented in Algorithm~\ref{algorithm_meshsipp_prune}. First, it evaluates $I_{\text{last}}$ -- if unexplored, the pruning check aborts immediately (Lines 8--10). Second, it jumps to $t_{\text{cutoff}}$ and scans backward toward $t_{\text{min}}$ (Lines 11--15). Finally, it verifies any remaining intervals between $t_{\text{cutoff}}$ and $I_{\text{last}}$, also in backward order (Lines 16--18). This permuted sequence dramatically minimizes average memory accesses. Because the entire window $W_m$ is comprehensively verified before a state is conclusively pruned (Line 19), the fail-fast heuristic strictly preserves the optimality guarantees of the algorithm.

\section{Experiments}

To evaluate the performance, scalability, and execution efficiency of MeshSIPP, we conducted both numerical benchmarks and ROS2/Gazebo simulations.

\subsection{Numerical experiments}
In the experiments we use the grid maps from the Moving AI benchmark \cite{sturtevant2012benchmarks}: \texttt{Denver\_1\_256} (city layout, $256 \times 256$), \texttt{arena} (open space with structured static obstacles, $49 \times 49$), and \texttt{empty\_64x64} (unobstructed grid, $64 \times 64$). Dynamic obstacles were simulated as independent agents moving along predefined trajectories. For each obstacle density, we generated sets of problem instances with randomly assigned valid start and goal configurations. In total, more than 6,000 evaluation runs were executed to ensure representativeness.

The motion primitive control set was constructed following the methodology of \cite{cubic} (Fig.~\ref{fig:base_control_set}). Robot orientations are discretized into 16 headings and for each one the control set provides 24 primitives with lengths ranging from 1 to 9 grid units, each terminating at a cell center and aligned with the valid discrete orientations. Motion along all primitives is executed at a constant speed of 0.1 cells per \textit{tick}, where a tick represents the base simulation time step.

\textbf{Baselines.} We compare MeshSIPP algorithm against four baselines that share the same underlying kinematic limits (maximum velocity and turning radius).

\begin{itemize}
\item \textbf{SIPP:} The standard Safe Interval Path Planning algorithm with motion primitives using the same control set from Fig.~\ref{fig:base_control_set}. Both SIPP and MeshSIPP were implemented in C++ to provide a meaningful comparison.
\item \textbf{ST-RRT*:} A spatiotemporal variant of the probabilistically optimal RRT* algorithm~\cite{ST-RRT}, connecting sampled states $(x, y, t)$ via constant-velocity linear segments. We utilized the standard implementation from the Open Motion Planning Library (OMPL v1.7.0 for C++) \cite{sucan2012the-open-motion-planning-library}. The time budget was set generously to $\max(1.0,~2 \times T_{SIPP})$ seconds, where $T_{SIPP}$ is the SIPP baseline runtime.
\item \textbf{ST-RRT*-Dubins:} An extension of the previous baseline where the spatial space is replaced with Dubins space. Consequently, the algorithm samples four-dimensional states $(x, y, \theta, t)$ and connects them via Dubins curves, natively respecting the agent's kinematic constraints. Due to the increased computational complexity of the steering function, the time budget was doubled relative to standard ST-RRT*.
\item \textbf{ORCA:} A two-stage approach widely used in robotics. A path is first planned ignoring dynamic obstacles. Then Optimal Reciprocal Collision Avoidance (ORCA) local planner \cite{vandenberg2011orca} tracks the path while reactively avoiding dynamic agents. We utilized the official implementation from the RVO2 library (also in C++) \cite{snape_rvo2}. Notably, while standard ORCA assumes cooperative collision avoidance where all agents share the responsibility equally, the dynamic obstacles in our setup are strictly non-cooperative. To account for this, we adjusted the ORCA velocity update factor, increasing it from the standard 0.5 to 1.0, thereby assigning the entire avoidance maneuver exclusively to the ego-agent.
\end{itemize}

\begin{figure*}[t]
\centering
\includegraphics[width=0.87\textwidth]{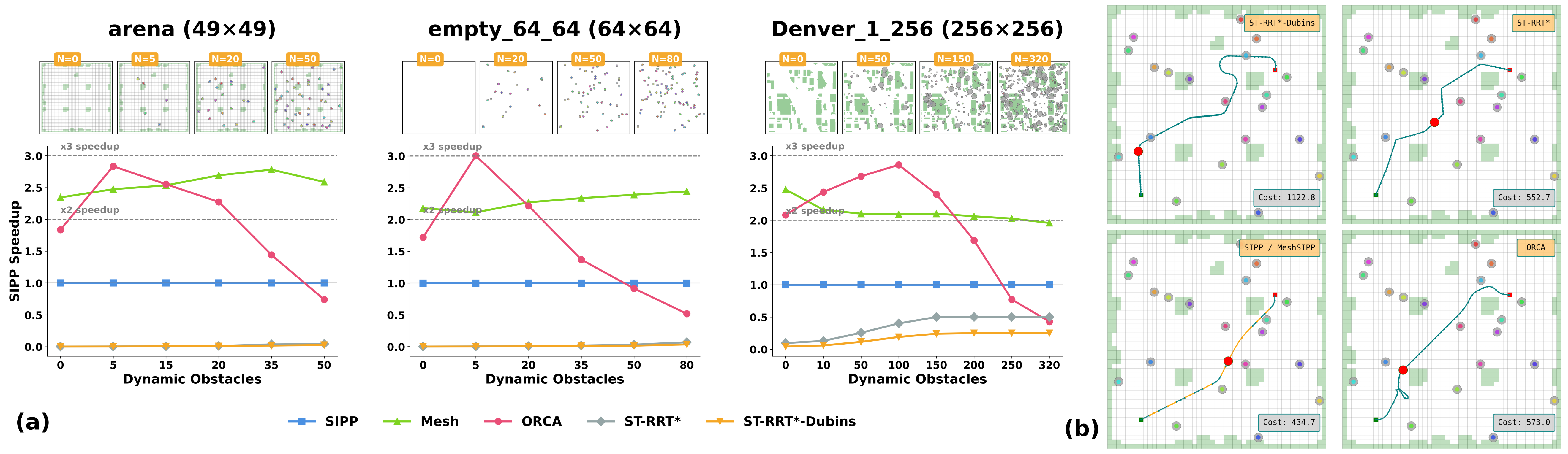}
\captionsetup{font=small}
\caption{Performance and trajectory evaluation of the tested algorithms. \textbf{(a)} Median runtime relative to standard SIPP across three map topologies with varying dynamic obstacle counts. \textbf{(b)} Snapshots of trajectory execution at $t = 250.0$ ticks on the \texttt{arena} map populated with 20 dynamic obstacles. The ego-agent (red circle, radius 1.0) navigates from start $(42, 7)$ to goal $(14, 37)$.}
\label{fig:algorithms_and_runtime_results}
\end{figure*}

\begin{table*}[t]
\centering
\captionsetup{font=small}
\caption{Median cost ratio and success rate (Cost / SR) relative to standard SIPP.}
\begin{tabular}{l | ccc | ccc | ccc}
\toprule
 & \multicolumn{3}{c}{\texttt{arena}} & \multicolumn{3}{c}{\texttt{empty\_64\_64}} & \multicolumn{3}{c}{\texttt{Denver\_1\_256}} \\
\textbf{Algorithm} & $N=0$ & $N=5$ & $N=35$ & $N=0$ & $N=20$ & $N=80$ & $N=0$ & $N=50$ & $N=200$ \\
\midrule
SIPP & 1.0 / 100\% & 1.0 / 100\% & 1.0 / 100\% & 1.0 / 100\% & 1.0 / 100\% & 1.0 / 100\% & 1.0 / 100\% & 1.0 / 100\% & 1.0 / 100\% \\
\textbf{MeshSIPP} & 1.0 / 100\% & 1.0 / 100\% & 1.0 / 100\% & 1.0 / 100\% & 1.0 / 100\% & 1.0 / 100\% & 1.0 / 100\% & 1.0 / 100\% & 1.0 / 100\% \\
ORCA & 1.0 / 100\% & 1.0 / 100\% & 1.3 / 90.2\% & 1.0 / 99.6\% & 1.1 / 100\% & 1.6 / 67.6\% & 1.0 / 98.5\% & 1.1 / 98.6\% & 1.3 / 70.1\% \\
ST-RRT* & 1.8 / 87.8\% & 1.6 / 87.4\% & 1.9 / 88.0\% & 1.7 / 91.8\% & 1.7 / 91.2\% & 2.0 / 90.3\% & 2.6 / 89.9\% & 2.2 / 89.7\% & 2.0 / 85.9\% \\
ST-RRT*-Dubins & 2.2 / 80.0\% & 2.3 / 79.3\% & 3.3 / 54.3\% & 1.8 / 89.3\% & 2.1 / 87.1\% & 3.9 / 68.8\% & 2.4 / 85.9\% & 2.4 / 84.2\% & 2.5 / 76.5\% \\
\bottomrule
\end{tabular}
\label{tab:cost_sr_results}
\end{table*}

\textbf{Results.} The experimental results are summarized in Fig.~\ref{fig:algorithms_and_runtime_results}(a) and Table~\ref{tab:cost_sr_results}. As anticipated, the MeshSIPP algorithm consistently achieves identical optimal path costs (1.0 ratio) and a 100\% success rate across all evaluated instances, matching the SIPP baseline. At the same time, it delivers a substantial reduction in computation time. The performance graphs demonstrate a stable speedup of approximately 2 to 3 times over the baseline SIPP. Crucially, this efficiency gain remains highly consistent across different map types and varying quantities of dynamic obstacles, demonstrating the scalability and reproducibility of MeshSIPP.

In contrast, the performance of ORCA depends heavily on obstacle density. Being a purely reactive local planner that bypasses complex graph searches, ORCA exhibits exceptional speed in sparse scenarios. In these simpler cases, ORCA frequently achieves a 100\% success rate and near-optimal trajectory costs. However, as the density of dynamic obstacles increases, its performance sharply degrades. In heavily populated environments, the high density of dynamic agents forces the reactive planner into continuous evasive maneuvers and oscillatory behavior, substantially increasing both the final path cost (execution time of the trajectory) and the total computation time required by ORCA to simulate and calculate the path itself. Because the obstacles are non-cooperative and strictly follow their predefined trajectories, they frequently corner the ego-robot into unavoidable collisions. This fundamental limitation is reflected in a drastically reduced success rate, dropping to 67.6\% on the \texttt{empty\_64\_64} map with 80 obstacles and 70.1\% on the \texttt{Denver\_1\_256} map with 200 obstacles.

For the sampling-based baselines, ST-RRT* and ST-RRT*-Dubins, execution time is bounded by a predefined budget, making direct runtime comparisons uninformative. Despite receiving generous computational budgets, both variants frequently fail to reach a 100\% success rate. Unlike ORCA, where a failure denotes a physical collision during execution, a failure in the sampling-based planners indicates an inability to discover a collision-free space-time trajectory before the time limit expires. When successful, the stochastic nature of RRT* inherently yields suboptimal paths, which is reflected in the elevated median cost ratios (up to $3.9\times$) in Table~\ref{tab:cost_sr_results}.

The runtime speedup profiles in Fig.~\ref{fig:algorithms_and_runtime_results}(a) provide additional context regarding the scalability of the sampling-based methods. Because the RRT* variants were allotted a minimum base budget of 1.0 second, their relative speedup appears near zero in simpler scenarios where the deterministic SIPP resolves the problem in milliseconds. However, in the most computationally demanding scenarios, such as the \texttt{Denver\_1\_256} map with extensive dynamic traffic, SIPP's runtime increases significantly. Consequently, the relative speedups of the ST-RRT* variants rise to meet their theoretical budget caps ($\frac{1}{2}$ and $\frac{1}{4}$ of SIPP's runtime). This convergence highlights that while sampling-based methods carry an excessive overhead for standard planar spatiotemporal navigation, their fixed-budget scaling becomes competitive as the combinatorial complexity of the exact search space approaches extremes, reinforcing their standard application in highly complex, high-dimensional configuration spaces~(such~as~multi-DOF~manipulators).

Sample trajectory executions and their total costs are illustrated in Fig.~\ref{fig:algorithms_and_runtime_results}(b). Indeed, SIPP and MeshSIPP produce smooth lowest-cost paths.

\subsection{Real-Time Validation in ROS 2 and Gazebo}

To validate MeshSIPP in a realistic physics engine, we deployed the planner within a Robot Operating System (ROS 2) \cite{Ros, Ros2} and Gazebo \cite{Gazebo} simulation. The physical environment and system architecture were strictly configured to test the planner's real-time adaptability.

\begin{itemize}
    
\item \textbf{Workspace:} A $10\times 10$ meter room, discretized into a $100\times 100$ occupancy grid. This yields a spatial resolution of $0.1$ meters per cell, utilized for both internal planning and RViz visualization.
\item  \textbf{Kinematics:} The planner utilizes the motion primitives from Fig.~\ref{fig:base_control_set}. Time resolution is fixed at 1 tick = 0.08 seconds. Consequently, the planned movement speed across grid cells is 0.1 meters per 0.08 seconds, equating to 1.25 m/s (approx. 4.5 km/h) -- a velocity representative of a fast pedestrian or a standard service robot.
\item  \textbf{Agents:} The ego-robot is modeled as a standard TurtleBot3 Waffle Pi (radius $0.22$ meters) \cite{TurtleBot3}. The environment is densely populated with 20 dynamic obstacles, modeled as moving cylinders with a radius of 10 cm and a height of 1.5 meters.

\item  \textbf{System Architecture:} A custom ROS 2 node manages dynamic obstacles, invokes MeshSIPP, publishes the generated trajectory to RViz, and tracks the resulting path via a PID wheel velocity controller. Standard localization is handled by the Nav2 \cite{Nav2} stack.
\end{itemize}

\begin{figure}[t]
\centering
\includegraphics[width=0.41\textwidth]{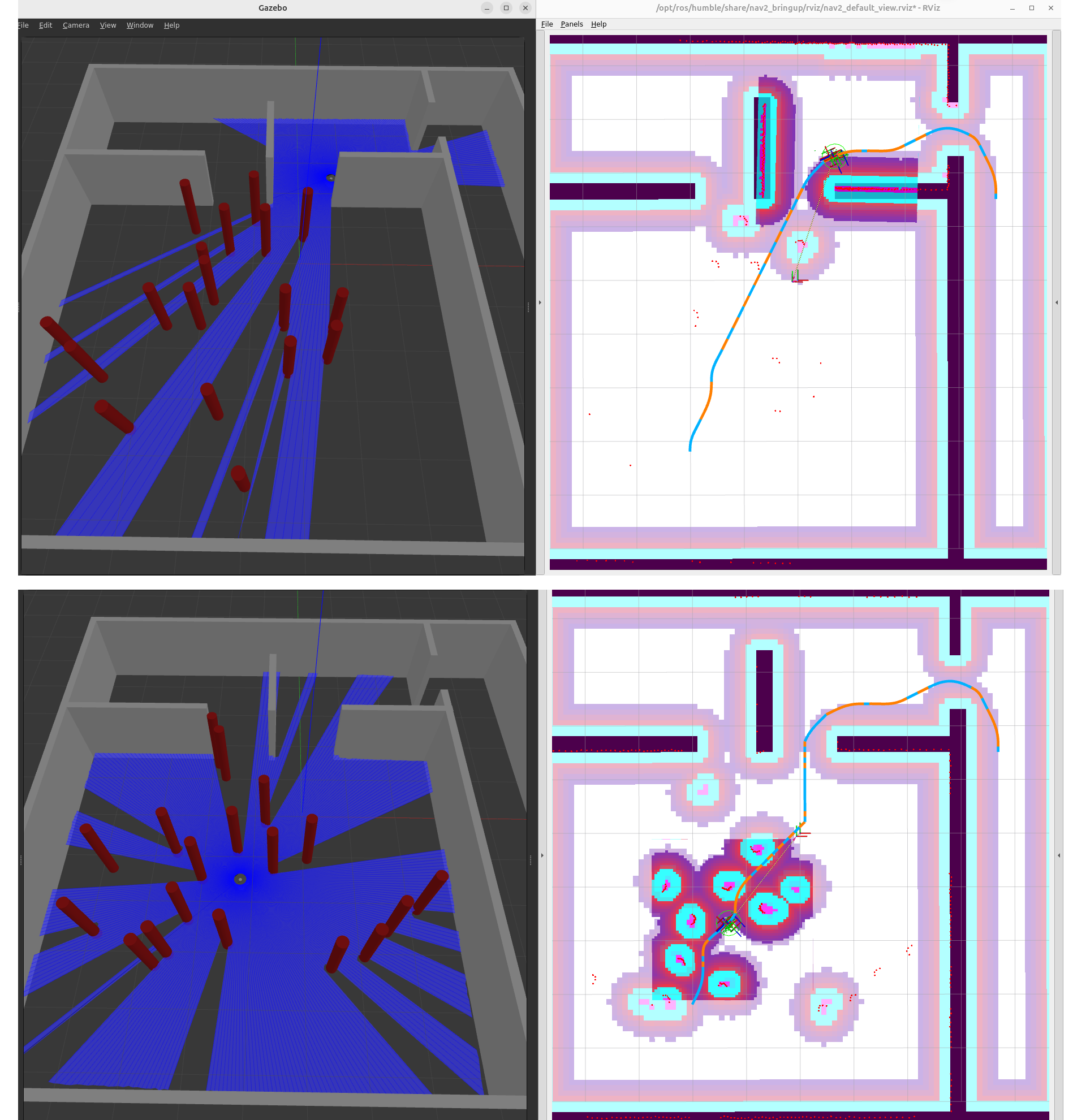}
\captionsetup{font=small}
\caption{Real-time navigation in the ROS 2 and Gazebo simulation. \textbf{Left:} The 3D physical environment in Gazebo, featuring the TurtleBot3 ego-agent, dynamic obstacles (red cylinders), and simulated 2D LiDAR scans (blue rays). \textbf{Right:} The corresponding RViz visualization showing the occupancy grid, inflated obstacle footprints, and the planned MeshSIPP trajectory (orange/blue primitives).}
\label{fig:gazebo_results}
\end{figure}

To bridge the gap between deterministic planning and real-world perception limits, we imposed a strict planning horizon. The robot is only permitted to predict dynamic obstacle trajectories for the next 200 ticks (16 seconds). Any time state beyond this 16-second threshold is assumed to be free of dynamic obstacles. To compensate for this limited foresight, the robot triggers a mandatory replan every 3 seconds. Throughout the simulation, MeshSIPP demonstrated high efficiency, requiring an average of only 50 ms of computation time for each replanning cycle.

By utilizing this receding-horizon strategy, the trajectory dynamically adapts and shifts over time as the ego-robot closes the distance to the moving obstacles. This allows the robot to seamlessly weave through dense, non-cooperative traffic and successfully reach the goal without any collisions.

This adaptive behavior is visualized in Fig.~\ref{fig:gazebo_results}. The top panels display the early stages of navigation, where distant obstacles fall beyond the temporal horizon and the robot plans a relatively direct route. The bottom panels illustrate the adapted trajectory later in the run, showcasing the planner heavily modifying the route to evade the approaching obstacle cluster. Full continuous execution of this scenario is available in the supplementary demonstration video.

\section{Conclusion}
In this paper, we presented MeshSIPP, an efficient spatiotemporal path planning algorithm that integrates state-lattice motion primitives with Safe Interval Path Planning (SIPP). By searching over a mesh graph with time-aware pruning, the method eliminates redundant branches while maintaining completeness and optimality. Evaluations demonstrate a consistent $3\times$ runtime speedup over standard lattice SIPP with identical trajectory quality.

Future work will extend MeshSIPP to multi-agent systems via decentralized or prioritized frameworks and validate the algorithm on physical robots in real-world scenarios.

\addtolength{\textheight}{-12cm}
\bibliographystyle{IEEEtran}
\bibliography{mybibfile}

\end{document}